\documentclass[review]{elsarticle}

\usepackage{bm}
\usepackage{amsmath}
\usepackage{amsfonts}
\usepackage{subcaption}
\usepackage{multirow}
\usepackage{hyperref}
\usepackage{lineno}

\newtheorem{lemma}{Lemma}

\newtheorem{proof}{Proof}

\journal{Journal of \LaTeX\ Templates}

\begin{document}

\begin{frontmatter}

\title{Stochastic Collapsed Variational Inference for Structured Gaussian Process Regression Network}
\tnotetext[mytitlenote]{Fully documented templates are available in the elsarticle package on \href{http://www.ctan.org/tex-archive/macros/latex/contrib/elsarticle}{CTAN}.}





\author[LBL1]{Rui Meng}
\author[UCSC]{Herbert K. H. Lee}
\author[LBL1,LBL2,UCB1,UCB2]{Kristofer Bouchard}
\address[LBL1]{Biological Systems and Engineering Division, Lawrence Berkeley National Laboratory}
\address[LBL2]{Scientific Data Division, Lawrence Berkeley National Laboratory}
\address[UCB1]{Helen Wills Neuroscience Institute, UC Berkeley}
\address[UCB2]{Redwood Center for Theoretical Neuroscience, UC Berkeley}
\address[UCSC]{University of California, Santa Cruz}

\begin{abstract}
This paper presents an efficient variational inference framework for deriving a family of structured gaussian process regression network (SGPRN) models. The key idea is to incorporate auxiliary inducing variables in latent functions and jointly treats both the distributions of the inducing variables and hyper-parameters as variational parameters. Then we propose structured variable distributions and marginalize latent variables, which enables the decomposability of a tractable variational lower bound and leads to stochastic optimization. Our inference approach is able to model data in which outputs do not share a common input set with a computational complexity independent of the size of the inputs and outputs and thus easily handle datasets with missing values. We illustrate the performance of our method on synthetic data and real datasets and show that our model generally provides better imputation results on missing data than the state-of-the-art. We also provide a visualization approach for time-varying correlation across outputs in electrocoticography data and those estimates provide insight to understand the neural population dynamics.
\end{abstract}


\end{frontmatter}


\section{Introduction}\label{sec:introduction}
Multi-output regression problems have arisen in various fields, including multivariate physiological time-series analysis \citep{durichen2014multitask}, chemometrics \citep{burnham1999latent}, and multiple-input multiple-output frequency nonselective channel estimation \citep{sanchez2004svm}. Often, the processes that generate such datasets are nonstationary. Modern instrumentation has resulted in ever increasing numbers of observations, as well as the occurrence of missing values. This motivates the development of scalable methods to forecast in such data sets.

Multi-ouput Gaussian process models or multivariate Gaussian process models (MGP) generalise the powerful Gaussian process predictive model to vector-valued random fields \citep{alvarez2010efficient,alvarez2011computationally}. Those models demonstrate improved prediction performance compared with the univariate Gaussian process because MGPs express correlation between outputs. Since the correlation information of data is encoded in the covariance function, modeling the flexible and computationally efficient cross-corvariance function is of interest. In the literature of MGPs, many approaches to building cross-covariance functions are based on combining univariate covariance functions. Specifically, those approaches can be classified into three categories: 
the linear model of coregionalization (LMC) \citep{bourgault1991multivariable, goulard1992linear} where the cross-covariance function is a linear function of valid stationary correlation functions, convolution techniques \citep{ver1998constructing, ver2004flexible, gneiting2010matern} where the cross-covariance is modeled as a process convolution of marginal covariance functions, and use of latent dimensions \citep{apanasovich2010cross} where it assumes the cross-covariance depends on latent dimensions. While convolution techniques and use of latent dimensions can utilize nonstationary kernels to enhance model flexibility, compared with LMC models, those models and parameters are hard to interpret and require Monte Carlo simulations, making inference in large datasets computationally expensive. 

To achieve both better interpretability and nonstationary behaviors, \cite{gelfand2004nonstationary,wilson2011gaussian,kleiber2012nonstationary,meng2021nonstationary} consider input-dependent coefficients in LMC. Such models can handle input-varying correlation across multivariate outputs. Especially for multivariate time series, \cite{meng2021nonstationary} propose a structured Gaussian process regression network (SGPRN) that captures time-varying scale, correlation and smoothness. Compared with the Gaussian process regression network \cite{wilson2011gaussian}, SGPRN employs stochastic lower triangular mixing coefficients with positive diagonal values and puts shared varying-lengthscale Gaussian processes for latent functions. It shows promising fitting and prediction performance on synthetic and electronic health records. However, due to the computation complexity of SGPRN, both maximum a posterior (MAP) and Monte Carlo Markov Chain (MCMC) inference is difficult to handle applications where either the number of observations and dimension size is large. Also those inference cannot be easily extended to incomplete datasets where part of outputs are missing.

We propose an efficient variational inference approach for SGPRN by employing the inducing variable framework on all latent processes \citep{titsias2010bayesian}, proposing a tractable variational bound amenable to doubly stochastic variational inference. We call our approach variational SGPRN (VSGPRN). This variational inference framework allows the model to handle missing data without increasing the computational complexity. We numerically provide evidence of the benefits of simultaneously modeling time-varying correlation, scale and smoothness in both a synthetic experiment and three different real-world problems. 

The main contributions of this work are threefold. 
\begin{itemize}
    \item Learning structured Gaussian process regression network using inducing variables on both mixing coefficients and latent functions.
    \item Employing doubly stochastic variational inference for structure Gaussian process regression network by performing exact marginalization of latent variables and constructing a tractable lower bound of log likelihood, allowing it suitable for mini-batching learning.
    \item Demonstrating that our proposed algorithm succeeds in handling time-varying correlation on missing data under different scenarios in both synthetic data and real datasets, and our method provides a visualization approach to understand dynamics of correlation and smoothness of data.
\end{itemize}

The structure of this paper is presented as follows: We first introduce related work in section~\ref{sec:related_work}. Then we review the SGPRN model \cite{kleiber2012nonstationary,meng2021nonstationary} in section~\ref{sec:model}. An efficient variational inference approach is proposed in Section~\ref{sec:inference}. Finally, our approach is illustrated on both synthetic experiments and three real datasets in Section~\ref{sec:experiments} and we delivery conclusions in Section~\ref{sec:conclusion}.


\section{Related Work}\label{sec:related_work}
Most multivariate gaussian process models build correlated outputs by mixing a set of independent latent processes. The mixing can be a linear combination with fixed coefficients \citep{goulard1992linear,seeger2005semiparametric,bonilla2008multi}. Those models are known as the linear coregionalization model (LMC) \citep{wackernagel2013multivariate} in the geostatistics literature. Based on the LMC structure, more sophisticated models are proposed. \cite{titsias2011spike} place a spike and slab prior over the coefficients. \cite{gelfand2004nonstationary,wilson2011gaussian,nguyen2013efficient,li2020scalable,meng2021nonstationary,meng2021bayesian} model more complex dependencies using input-dependent coefficients. 

The Gaussian process regression network (GPRN) is proposed in  \cite{wilson2011gaussian} and efficient inference approaches are studied in \cite{nguyen2013efficient,li2020scalable}. GPRN is a linear coregionalization model with input-dependent coefficients and the coefficients across time for all elements of the coefficient matrix are modeled by independent stationary Gaussian processes. However, GPRN is not identifiable for the coefficients since the decomposition of covariance matrix is not unique \cite{meng2021nonstationary}. This makes model interpretation challenging at best. To tackle with this issue, \cite{gelfand2004nonstationary} consider a matrix-variate spatial Wishart process and more efficient models are proposed by directly putting constraints on coefficients in \cite{guhaniyogi2013modeling,meng2021nonstationary}. 

Our work develops an efficient variational inference algorithm for SGPRN by taking advantage of inducing variables and stochastic variational inference for scalability. Inducing variables are the key catalyst for achieving sparsity in Gaussian process models in \cite{titsias2010bayesian,Hensman_2013}. Moreover, \cite{nguyen2014collaborative} claims that sharing "sparsity structure" is not only a reasonable assumption, but a crucial component when modeling multi-output data. Stochastic variational inference plays a important role in scalable inference and has already demonstrated it efficiency in various models including deep Gaussian process \cite{salimbeni2017doubly} and neural process \cite{wang2020doubly}.

\section{Structured Gaussian Process Regression Network} \label{sec:model}
\begin{figure*}[ht!]
    \centering
    \includegraphics[width = \textwidth]{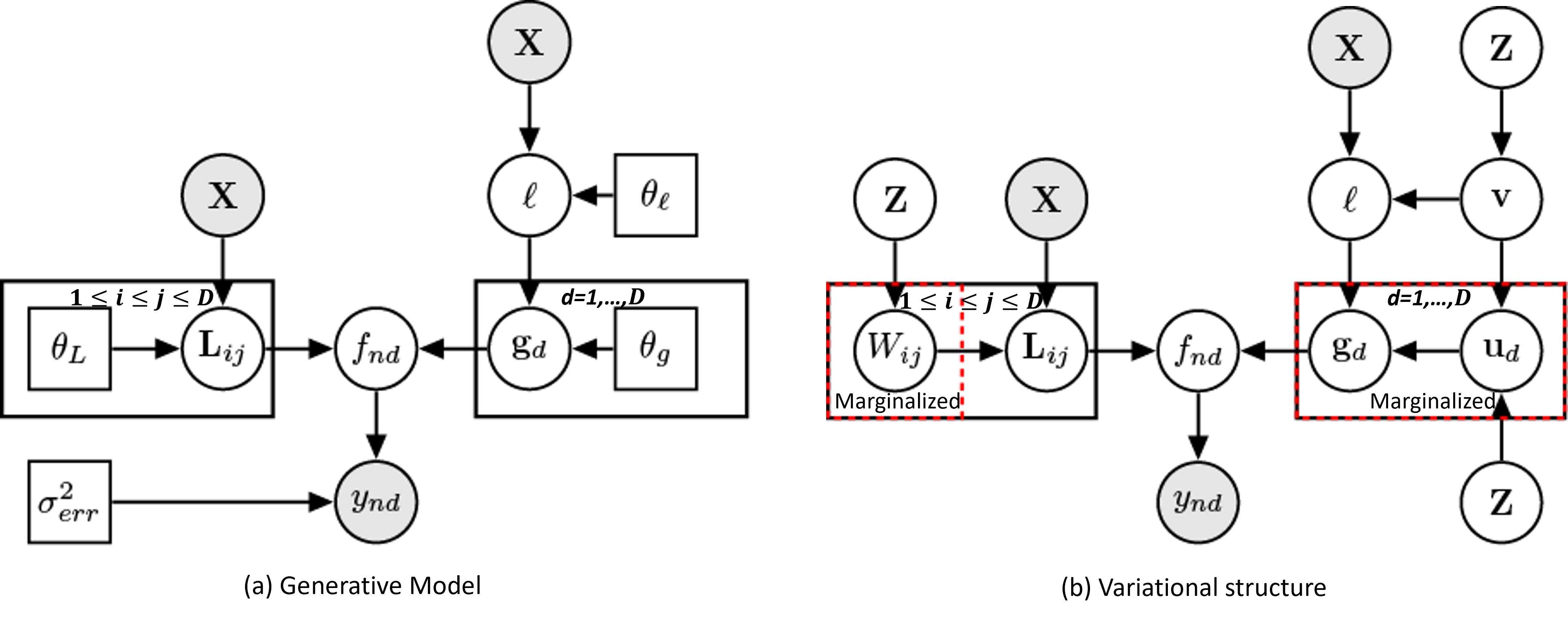}
    \caption{Graphical model of structured Gaussian process regression network. Left: Illustration the generative model. Right: Illustrate of the variational structure. Dashed red block means that we marginalize out those latent variables in variational inference framework.}
    \label{fig:graphical_model}
\end{figure*}

Assume $\bm y(\bm x) \in \mathbb{R}^D$ is a vector-valued function of $\bm x \in \mathbb{R}^P$, where $D$ is the dimension size of outputs and $P$ is the dimension size of inputs. The structured Gaussian process regression network (SGPRN) model assumes that noisy observations $\bm y(\bm x)$ are the linear combination of latent variables $\bm g(\bm x) \in \mathbb{R}^D$, corrupted by Gaussian noise $\bm \epsilon(\bm x)$. The coefficients $\bm L(\bm x) \in \mathbb{R}^{D \times D}$ of the latent functions are assumed to be a stochastic lower triangular matrix with positive values on the diagonal for model identification \cite{guhaniyogi2013modeling,meng2021nonstationary}. Thus, the SGPRN is defined in Figure~\ref{fig:graphical_model} and is shown as follows:
\begin{align}
 \bm y(\bm x) & = \bm f(\bm x) + \bm \epsilon(\bm x), \quad \bm f(\bm x) = \bm L(\bm x)\bm g(\bm x), \nonumber \\ 
 \bm\epsilon(\bm x) & \stackrel{iid}{\sim} \mathcal{N}(0, \sigma^2_{err}I)\,. \label{eq:TVLMC}
\end{align}

Moreover, each latent function $g_d$ in $\bm g$ is independently sampled from a Gaussian process (GP) with a non-stationary kernel $K^g$ and the stochastic coefficients are modeled via a structured GP based prior as proposed in \cite{guhaniyogi2013modeling} with a stationary kernel $K^l$ such that
\begin{align}
g_d & \stackrel{iid}{\sim} \mathrm{GP}(0, K^g)\,, \qquad d=1, \ldots, D\,, \nonumber \\
l_{ij} & \sim \begin{cases}
                 \mathrm{GP}(0, K^l)\,, & i > j\,,  \\
                 \mathrm{logGP}(0, K^l)\,, & i = j\,,
                 \end{cases} 
\label{eq:prior_GP}
\end{align}
where $\mathrm{logGP}$ denotes the log Gaussian process \citep{moller1998log}. $K^g$ is modelled as a Gibbs correlation function
\begin{align}
    K^g(\bm x, \bm x')  & = \sqrt{\frac{2\ell(\bm x)\ell'(\bm x)}{\ell(\bm x)^2 + \ell(\bm x')^2}}\exp\left(-\frac{\|\bm x-\bm x'\|^2}{\ell(\bm x)^2 + \ell(\bm x')^2}\right)\,, \nonumber \\
    \ell & \sim \mathrm{logGP}(0, K^{\ell})\,, \nonumber 
\end{align}
 where $\ell$ determines the input-dependent length scale of the shared correlations in $K^g$ for all latent functions $g_d$. This varying length-scale process $\ell$ plays an important role to model nonstationary time series illustrated in \cite{remes2017non, meng2021nonstationary}.

Given the stochastic coefficients $\bm L$, the cross covariance function of $\bm f(\bm x)$ is
\begin{align}
    K^f(\bm x, \bm x', m, m') = K^g(\bm x, \bm x') \bm l_{m}(\bm x)\bm l_{m'}^T(\bm x') \,, \label{eq:kernel}
\end{align} 
where $\bm l_{m}(\bm x)$ denotes the $m^{\text{th}}$ row of $\bm L(\bm x)$. With a deterministic coefficient matrix $\bm L(\bm x) \equiv \bm L$ and a stationary GP for all latent processes $\bm g$, this model is equivalent to the intrinsic coregionalization model \citep{goulard1992linear}. 

Let $\bm X = \{\bm x_i\}_{i=1}^N$ be the set of observed inputs and $\bm Y = \{\bm y_i\}_{i=1}^N$ be the set of observed outputs. Denote $\bm \eta$ as the concatenation of all coefficients and all log length-scale parameters, i.e., $\bm \eta = (\bm l, \tilde{\bm \ell})$ evaluated at training inputs $\bm X$. Here, $\bm l$ is a vector including the entries below the main diagonal and the entries on the diagonal in the log scale and $\tilde{\bm \ell} = \log \bm \ell$ is the length-scale parameters in log scale. Also, denote $\bm \theta = (\bm \theta_l, \bm\theta_\ell, \sigma^2_{err})$ as all hyper-parameters, where $\bm \theta_l$ and $\bm\theta_\ell$ are the hyper-parameters in kernel $K_l$ and $K_\ell$. According to the model specification in (\ref{eq:TVLMC}), the prior over $\bm \eta$ is a $N(D(D+1)/2+1)$ dimensional multivariate Gaussian distribution with a block diagonal covariance matrix $p(\bm \eta| \bm \theta_l, \bm \theta_\ell) = \mathcal{N}(\bm 0, \bm C_\eta)$
where the first $M(M+1)/2$ blocks of $\bm C_\eta$ are induced by the kernel $K_l$ and the last one block is induced by the kernel $K_\ell$.

Given model parameters $\bm \eta$ and hyper-parameters $\bm \theta$, by marginalizing the latent function $\bm g$, the conditional likelihood is $
p(\bm Y|\bm \eta, \sigma^2_{err}) = \mathcal{N}(\bm y| \bm 0, \bm K^f + \sigma^2_{err}\bm I)$ where $\bm K^f$ is the covariance function $K^f$ in (\ref{eq:kernel}) evaluated at training inputs $\bm X$. Hence, the main inference task in the SGPRN is maximum \textit{a posteriori} by maximizing the posterior
\begin{align}
    p(\bm \eta|\bm Y, \bm \theta) \propto p(\bm Y|\bm \eta, \sigma^2_{err})p(\bm \eta| \bm \theta_l, \bm \theta_\ell)\,, \label{eq:pos_NMGP}
\end{align}
which is computationally intractable in general because the computational complexity of $p(\bm \eta|\bm Y, \bm \theta)$ is $\mathcal{O}(N^3D^3)$. To overcome this issue, we propose an efficient variational inference to significantly reduce the computational burden in the next section.

\section{Inference} \label{sec:inference}
Despite the success of the existing SGPRN inference methods, current inference methods are only available for complete data and the computational cost of inference is prohibitive for massive high-dimensional outputs that are common in many real-world datasets. To alleviate the computational burden associated with (\ref{eq:pos_NMGP}), we introduce a shared set of inducing inputs $\bm Z = \{\bm z_m\}_{m=1}^M$ that lie in the same space as the inputs $\bm X$ and a set of shared inducing variables $\bm w_d$ for each latent function $g_d$ evaluated at the inducing inputs $\bm Z$. Likewise, we consider inducing variables $\bm u_{ii}$ for function $\log L_{ii}$ when $i = j$, $\bm u_{ij}$ for function $L_{ij}$ when $i > j$, and inducing variables $\bm v$ for function $\log \ell(\bm x)$ evaluated at inducing inputs $\bm Z$. We denote those collective variables as $\bm l=\{\bm l_{ij}\}_{i\geq j}$, $\bm u = \{\bm u_{ij}\}_{i\geq j}$, $\bm g = \{\bm g_d\}_{d=1}^D$, $\bm w = \{\bm w_d\}_{d=1}^D$, $\bm \ell$ and $\bm v$. Then we redefine the model parameters $ \bm \eta= (\bm l, \bm u, \bm g, \bm w, \bm \ell, \bm v)$, and the prior of those model parameters is
\begin{align}
    p(\bm \eta) = p(\bm l| \bm w)p(\bm w)p(\bm g|\bm u, \bm\ell, \bm v)p(\bm u)p(\bm \ell|\bm v)p(\bm v)\,. \label{eq:prior}
\end{align}

The core assumption of inducing point-based sparse inference is that the inducing variables are sufficient statistics for the training and testing data in the sense that training and testing data are independent given the inducing variables. In the context of our model, it suggests that the posterior processes of $L$, $g$ and $\ell$ are sufficiently determined by the posterior distribution of $\bm u$, $\bm w$ and $\bm v$. To conduct the variational inference, we propose structured variational distributions in Section~\ref{sec:SVD}. Given the proposed structured variational distributions, we derive the evidence lower bound (ELBO) in Section~\ref{sec:VELB}. Due to the nonconjugaty of this model, instead of doing expectation in ELBO, we perform the marginalization on inducing variables $\bm u$, $\bm w$ and $\bm g$, and then use the reparameterization trick to apply end-to-end training with stochastic gradient descend in Section~\ref{sec:learning}. We provide the prediction procedure in Section~\ref{sec:prediction} and discuss the inference procedure for missing data in Section~\ref{sec:missing}.

\subsection{Structured Variational Distribution} \label{sec:SVD}
To capture the posterior dependency between the latent functions, we propose a structured variational distribution of model parameters $\bm \eta$ used to approximate its posterior distribution as
\begin{align}
    q(\bm \eta) & = p(\bm l| \bm u) p(\bm g| \bm w, \bm \ell, \bm v) p(\bm \ell| \bm v) q(\bm u, \bm w, \bm v)\,. \label{eq:var_dist}
\end{align}
This variational structure is illustrated in Figure~\ref{fig:graphical_model}. The variational distribution of inducing variables $q(\bm u, \bm w, \bm v)$ fully characterizes the distribution of $\bm q(\bm \eta)$. Thus, the inference of $q(\bm u, \bm w, \bm v)$ is of interest. Furthermore, we assume the parameters $\bm u$, $\bm w$, and $\bm v$ are Gaussian and mutually independent,
\begin{align}
    q(\bm u, \bm w, \bm v) = \prod_{i \geq j} \mathcal{N}(\bm u_{ij}| \bm m^u_{ij}, S^u_{ij}) \prod_{d=1}^D \mathcal{N}(\bm w_d| \bm m^w_{d}, S^w_{d}) \mathcal{N}(\bm v|\bm m^v, S^v)\,,
\end{align}

Given the definition of Gaussian process priors in (\ref{eq:prior_GP}), the conditional distributions $p(\bm l| \bm u)$, $p(\bm g| \bm w, \tilde{\bm \ell}, \bm v)$, and $p(\bm \ell| \bm v)$ have closed-form expressions as follows

\begin{align}
    p(\bm l|\bm u) & = \prod_{i = j} \log\mathcal{N}(\bm l_{ii}| \bm \mu_{ii}^l, \Sigma_{ii}^l) \prod_{i > j}\mathcal{N}(\bm l_{ij}| \bm \mu_{ij}^l, \Sigma_{ij}^l) \,, \label{eq:cond_l} \\
    p(\bm g| \bm w, \bm \ell, \bm v) & = \prod_{d=1}^D \mathcal{N}(\bm g_d| \bm \mu_d^g, \Sigma_d^g) \,, \label{eq:cond_g}\\
    p(\bm \ell| \bm v) & = \log\mathcal{N}(\bm \ell| \bm \mu^\ell, \Sigma^\ell) \,. \label{eq:cond_ell}
\end{align}
       
The derivation of the conditional mean and conditional covariance matrix is available in Appendix A.1.

\subsection{Variational Evidence Lower Bound}\label{sec:VELB}
The evidence lower bound (ELBO) of the log likelihood of observations under our structured variational distribution $q(\bm \eta)$ is derived using Jensen's inequality as:
\begin{align}
    \log p(\bm Y) & = \log \int p(\bm Y, \bm \eta)d \bm \eta \nonumber \\
    & \geq E_{q(\bm \eta)}\left[\log\left(\frac{p(\bm Y| \bm g, \bm l)p(\bm u)p(\bm w)p(\bm v)}{q(\bm u, \bm w, \bm v)}\right)\right]  \nonumber \\
    & = \sum_{n = 1}^N \sum_{d = 1}^D E_{q(\bm g_n, \bm l_n)}\log(p(y_{nd} | \bm g_n, \bm l_n)) + A\,, \label{eq:ELBO}
\end{align}
where $A = \mathrm{KL}(q(\bm u)||p(\bm u)) + \mathrm{KL}(q(\bm w)||p(\bm w)) + \mathrm{KL}(q(\bm v)||p(\bm v))$ is a regularization term, $\bm g_n = \{g_{dn} = (\bm g_d)_n\}_{d=1}^D$ and $\bm l_n = \{l_{ijn} = (\bm l_{ij})_n\}_{i\geq j}$.

The structured decomposition \eqref{eq:var_dist} has been used by \cite{titsias2010bayesian} and \cite{Hensman_2013} to derive variational inference for the single output case and it is also used by \cite{nguyen2014collaborative} for a multivariate output case. The benefit of this structure is that the conditional distributions, \eqref{eq:cond_l}, \eqref{eq:cond_g} and \eqref{eq:cond_ell} are cancelled in the derivation of the lower bound in \eqref{eq:ELBO}, which alleviates the computational burden of inference.  Because the first term in (\ref{eq:ELBO}) shows that the observations are all conditionally independent given $\bm g$ and $\bm l$, the lower bound decomposes across both inputs and outputs and this enables the use of stochastic optimization methods. Moreover, since $q(\bm u)$, $p(\bm u)$, $q(\bm w)$, $p(\bm w)$, $q(\bm v)$ and $p(\bm v)$ are all multivariate Gaussian distributions, the KL divergence terms are analytically tractable. The challenge is to solve for the individual expectations because it is intractable to derive the marginal posterior of $\bm g$ and $\bm l$. Therefore, instead of doing expectation, we consider stochastic inference which requires sampling $\bm g$ and $\bm l$ from their variational distribution. The sampling based learning approach is provided in the next section.

\subsection{Learning the parameters of the Variational Distribution} \label{sec:learning}

To achieve efficient sampling for $\bm l$ and $\bm g$ from the variational posterior $q(\bm \eta)$, we marginalize unnecessary intermediate variables. By marginalizing the inducing variables $\bm u$ and  $\bm w$, we obtain the marginal distributions 
\begin{align}
    q(\bm l) & = \prod_{i=j}\log\mathcal{N}(\bm l_{ii}| \tilde{\bm \mu}_{ii}^l, \tilde{\Sigma}_{ii}^l)\prod_{i>j}\mathcal{N}(\bm l_{ij}| \tilde{\bm \mu}_{ij}^l, \tilde{\Sigma}_{ij}^l) \,, \label{eq:marg_l} \\
    q(\bm g| \bm \ell, \bm v) & = \prod_{d=1}^D\mathcal{N}(\bm g_d|\tilde{\bm \mu}_d^{g}, \tilde{\Sigma}_d^{g}) \,, \label{eq:marg_g}
\end{align}
with a joint distribution $q(\bm \ell, \bm v) = p(\bm \ell|\bm v)q(\bm v)$, where the conditional mean and covariance matrix are derived in Appendix A.2.

According to the marginal distributions (\ref{eq:marg_l}) and (\ref{eq:marg_g}), the marginal distributions for latent factors $\bm g_n$ and coefficients $\bm l_n$ in (\ref{eq:ELBO}) are derived as
$
q(\bm l_n) = \prod_{i = j}\log\mathcal{N}(l_{iin}| \tilde{\mu}_{iin}^l, \tilde{\sigma}^{l2}_{iin}) \prod_{i > j} \mathcal{N}(l_{ijn}| \tilde{\mu}_{ijn}^l,  \tilde{\sigma}^{l2}_{ijn})$ and $
q(\bm g_n| \bm \ell, \bm v) = \prod_{d = 1}^D\mathcal{N}(g_{nd}| \tilde{\mu}_{dn}^g, \tilde{\sigma}^{g2}_{dn})$
where $\tilde{\sigma}^{l2}_{ijn}$ and $\tilde{\sigma}^{g2}_{dn} = \tilde{\Sigma}^g_{ijnn}$ are the $n^\mathrm{th}$ diagonal element of $\tilde{\Sigma}^l_{ij}$ and $\tilde{\Sigma}^g_{ij}$ respectively.


Moreover ,we marginalize the latent variables $\bm g_n$ and then the individual expectation is
\begin{align}
    & \mathrm{E}_{q(\bm g_n, \bm l_n)}\log(p(y_{nd}|\bm g_n, \bm l_n)) \nonumber \\
    & = \int\left(\log\mathcal{N}(y_{nd}|\sum_{j=1}^D l_{djn}\tilde{\mu}^g_{jn}, \sigma^2_{err}) - \frac{1}{2\sigma^2_{err}}\sum_{j=1}^D l_{djn}^2\tilde{\sigma}^{g2}_{jn}\right) \nonumber \\
    & \quad q(\ell_n, \bm v)q(\bm l_{d\cdot n})d(\bm l_{d\cdot n}, \ell_n, \bm v))\,. \label{eq:expectaion_n_v1}
\end{align}
The details of the marginalization are available in Appendix A.3.

Directly evaluating the ELBO is still challenging due to the non-linearities introduced by our structured prior. Recent progress in black box variational inference \citep{ranganath2014black,kingma2013auto,rezende2014stochastic, titsias2014doubly}
avoids this difficulty by computing noisy unbiased estimates of the gradient of ELBO, via approximating the expectations with unbiased Monte Carlo estimates and relying on either score function estimators \citep{ranganath2014black} or reparameterization gradients \citep{kingma2013auto,rezende2014stochastic,titsias2014doubly} to differentiate through a sampling process. In practice, reparameterization gradients exhibit significantly lower variances than score function estimators \citep{ghosh2018structured}. In this work, we leverage the reparameterization gradients for (\ref{eq:expectaion_n_v1}) to separate the source of randomness from the parameters with respect to which gradients are sought. Typically, for Gaussian variational approximation, the well known non-centered parameterization, $\xi \sim \mathcal{N}(\mu, \sigma^2) \Longleftrightarrow \epsilon\sim\mathcal{N}(0,1), \xi = \mu + \sigma \epsilon$, allows us to compute the Monte Carlo gradients.

The details of the reparameterization are available in Appendix A.4. Suppose all independent normal variables in the reparameterization are denoted by $\bm z$, then the Monte Carlo gradients are
\begin{align}
    &\quad \nabla_{\eta} \mathrm{E}_{q(\bm g_n, \bm l_n)}\log(p(y_{nd}|\bm g_n, \bm l_n)) \nonumber \\
    & = \frac{1}{S}\sum_s \nabla_{\eta}\log(p(y_{nd} | \bm g_n^{(s)}, \bm l_{d\cdot n}^{(s)}))\,, \label{eq:ELBO_grad}
\end{align}
where $S$ is the number of samples. $\bm g_n^{(s)}$ and $\bm l_{d\cdot n}^{(s)}$ depend on the randomness from $\bm z^{(s)}\sim \mathcal{N}(\bm 0, I)$.

Note that evaluating ELBO (\ref{eq:ELBO}) involves two sources of stochasticity. First, we approximate the expectation in (\ref{eq:expectaion_n_v1}) to compute the unbiased estimates of the gradients of ELBO (\ref{eq:ELBO_grad}). Second, since ELBO (\ref{eq:ELBO}) factorizes over observations, it allows us to approximate the bound with data sub-sampling stochasticity \citep{hoffman2010online, hoffman2013stochastic}.  On the other hand, all hyper-parameters $\bm \theta$ are allowed to be optimized in the stochastic optimization.

\subsection{Prediction} \label{sec:prediction}
Model prediction depends on the inferred variational distribution $q(\bm u, \bm w, \bm v)$. Given a new input $\bm x^*$, predictive distributions on the latent processes are obtained through the following sampling procedures. 
We first sample the length-scale parameters on both training inputs $\bm X$ and a new input $\bm x^*$ from $q(\bm \ell, \ell^*) = \int p(\bm \ell, \ell^*|\bm v)q(\bm v)d\bm v$. We denote the $s^{\mathrm{th}}$ samples as $\bm \ell^{(s)}$ and $\ell^{*(s)}$ respectively. Conditional on them, we sample the latent process $\bm g^{*(s)}$ at input $\bm x^*$ from $q(g^*_d|\bm \ell^{(s)}, \ell^{*(s)}) = \int p(g^*_d|\bm \ell^{(s)}, \ell^{*(s)}, \bm w)q(\bm w)d\bm w$. We sample the coefficients $\bm l^{*(s)}$ at input $\bm x^*$ from $q(l_{ij}^{*}) = \int p(l_{ij}^*|\bm u_{ij})q(\bm u_{ij})d\bm u_{ij}$. Finally, given $\bm g^{*(s)}$ and $\bm l^{*}$ we sample the observations $\bm y^{*(s)}$ at input $\bm x^*$ througn the linear mixing mechanism from $q(\bm y^*|\bm g^{*(s)}, \bm l^{*(s)}) = \int p(\bm y^*|\bm g^{*(s)}, \bm l^{*(s)}, \bm \epsilon^*)p(\bm \epsilon^*)d\bm \epsilon^*$.

\subsection{Inference for Missing Data} \label{sec:missing}

Because in ELBO (\ref{eq:ELBO}), observations $\{y_{nd}\}$ are mutually conditional independent on all the model parameters $\bm \eta$ and hyper-parameters $\bm \theta$, instead of summing up the likelihoods of complete data we take the sum of the individual likelihoods of observed data to compute the ELBO. The gradients of the ELBO are estimated by summing up the individual Monte Carlo gradients \eqref{eq:ELBO_grad} over all observed data.

\section{Experiments} \label{sec:experiments}
This section illustrates the performance of our model with numerical results. In particular, we focus on multivariate time series where the input dimension is one. We first show that our approach can model the time-varying correlation and smoothness of outputs on 2D synthetic datasets in three scenarios with respect to different types of frequencies but the same missing data mechanism. Then we compare the imputation performance on missing data with other inducing-variable based sparse multivariate Gaussian process models on two real datasets. Finally, we explore the dynamics of correlation of neuronal activities from different channels using electrocorticography data. All experiments are run on an Ubuntu system with Intel(R) Core(TM) i7-7820X CPU @ 3.60GHz and 128G memory. 

\subsection{Synthetic Experiments}
\begin{table*}[ht!]
    \centering
    \caption{Prediction measurements on three synthetic datasets and different models. LF denotes the low-frequency dataset, HF the high-frequency dataset, and VF the time-varying dataset. Three prediction measures are provided. RMSE is root mean square error, ALCL is average length of confidence interval, and CR is coverage rate. All three measurements are summarized by the mean and standard deviation across 10 runs with different random initializations, i.e., 2.25(1.33e-13) denotes mean 2.25 with standard deviation 1.33e-13.}
    \label{tab:sim}
    \begin{tabular}{|c|c|c|c|c|}
    \hline
    Data & Model & RMSE & ALCI & CR \\
    \hline
    \multirow{4}{*}{LF} & IGPR \cite{Rasmussen_2004} & 2.25(1.33e-13) & 2.18(1.88e-13) & 0.835(0) \\ \cline{2-5}
    & ICM \cite{wackernagel2013multivariate} & 2.26(2.54e-5) & 2.18(1.22e-5) & 0.835(0)  \\ \cline{2-5}
    & CMOGP \cite{nguyen2014collaborative} & 1.43(6.12e-2) & 1.36(1.98e-1) & 0.651(3.00e-2) \\ \cline{2-5}
    & VGPRN \cite{nguyen2013efficient} & 1.01(0.31) & - & - \\ \cline{2-5}
    \cline{2-5}
    & VSGPRN & \textbf{1.00(1.43e-1)} & 2.21(6.56e-2) & \textbf{0.892(1.63e-2)} \\ \cline{2-5}
    \hline
    \multirow{4}{*}{HF} & IGPR \cite{Rasmussen_2004} & 1.51(6.01e-14) & 3.17(1.30e-13) & 0.915(2.22e-16) \\ \cline{2-5}
    & ICM \cite{wackernagel2013multivariate} & 1.52(1.01e-5) & 3.17(1.19e-5) & 0.910(0) \\ \cline{2-5} 
    & CMOGP \cite{nguyen2014collaborative} & 1.29(3.04e-2) & 2.34(3.31e-1) & 0.729(3.07e-2) \\ 
    \cline{2-5}
    & VGPRN \cite{nguyen2013efficient} & 1.11(0.25) & - & - \\ \cline{2-5}
    \cline{2-5}
    & VSGPRN  &\textbf{ 1.10(1.98e-1)} & 2.74(7.94e-2) & \textbf{0.930(1.14e-2)} \\ \cline{2-5}
    \hline
    \multirow{4}{*}{VF}& IGPR \cite{Rasmussen_2004} & 1.64(8.17e-14) & 3.19(3.02e-13) & 0.875(0) \\ \cline{2-5}
    & ICM \cite{wackernagel2013multivariate} & 1.66(2.37e-3) & 3.16(1.49e-3) & 0.880(1.50e-3) \\ \cline{2-5} 
    & CMOGP \cite{nguyen2014collaborative} & 2.24(3.08e-1) & 2.56(9.29e-1) & 0.697(1.56e-1) \\ \cline{2-5}
    & VGPRN \cite{nguyen2013efficient} & 1.04(0.67) & - & - \\ \cline{2-5}
    \cline{2-5}
    & VSGPRN  & \textbf{1.24(1.33e-1)} & 2.92(1.21e-1) & \textbf{0.887(9.80e-3)} \\ \cline{2-5}
    \hline
    \end{tabular}
\end{table*}

We conduct experiments on three synthetic time series with low frequency (LF), high frequency (HF) and varying frequency (VF) respectively. They are generated from the system of equations
\begin{align}
y_1(t) & = 5\cos(2\pi wt^{s}) + \epsilon_1(t)\,, \nonumber \\
y_2(t) & = 5(1-t)\cos(2\pi wt^{s}) - 5t\cos(2\pi wt^{s})  + \epsilon_2(t)\,, \label{eq: sim}
\end{align}
where $\{\epsilon_i(t)\}_{i=1}^2$ are independent standard white noise processes. The value of $w$ refers to the frequency and the value of $s$ characterizes the smoothness. The LF and HF datasets use the same $s=1$, implying the smoothness is invariant across time. But they employ different frequencies,  $w=2$ for LF and $w=5$ for HF (i.e., two periods and five periods in a unit time interval respectively). The VF dataset takes $s=2$ and $w=5$, so that the frequency of the function is gradually increasing as time increases. For all three datasets, the system $(\ref{eq: sim})$ shows that as time $t$ increases from $0$ to $1$, the correlation between $y_1(t)$ and $y_2(t)$ gradually varies from positive to negative. Within each dataset, we randomly select 200 training data, in which 100 time stamps are sampled on the interval $(0, 0.8)$ for the first dimension and the other 100 time stamps sampled on the interval $(0.2, 1)$ for the second dimension. For the test inputs, we randomly select 100 time stamps on the interval $(0, 1)$ for each dimension. 

\begin{figure}[!ht]
    \bigskip
    \noindent
    \begin{minipage}{0.68\textwidth}
    \begin{subfigure}{\linewidth}
    \subcaption{} \label{subfig:upper-left}
    \vspace{-3mm}
    \includegraphics[width=0.32\linewidth]{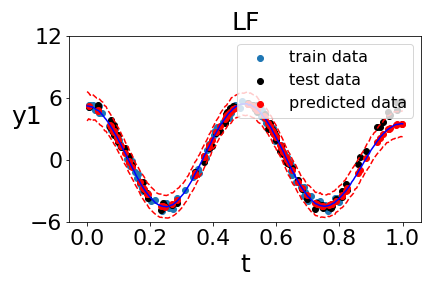}
    \includegraphics[width=0.32\linewidth]{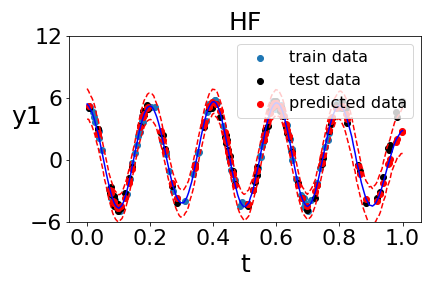}
    \includegraphics[width=0.32\linewidth]{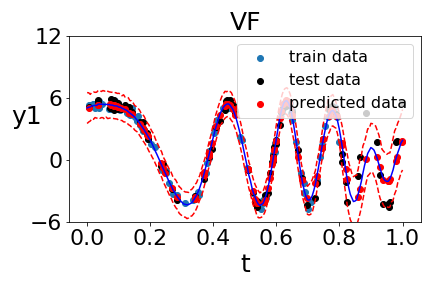}
    \end{subfigure}
    \begin{subfigure}{\linewidth}
    \subcaption{} \label{subfig:lower-left}
    \vspace{-3mm}
    \includegraphics[width=0.32\linewidth]{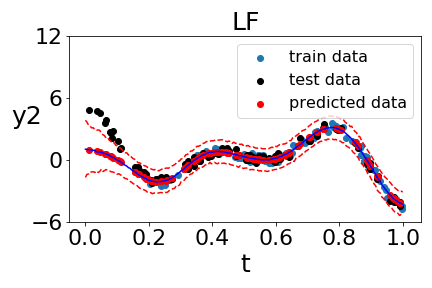}
    \includegraphics[width=0.32\linewidth]{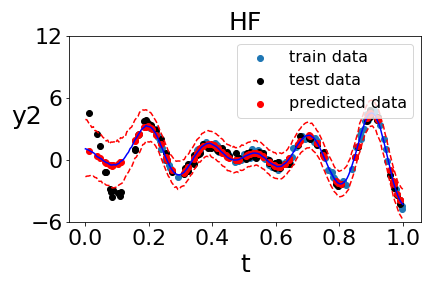}
    \includegraphics[width=0.32\linewidth]{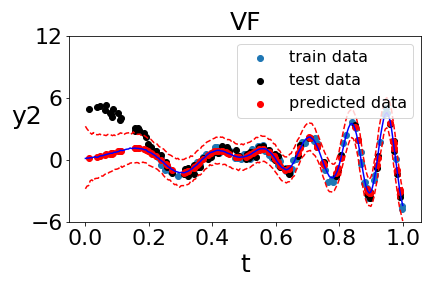}
    \end{subfigure}
    \end{minipage}
    \hfill
    \noindent
    \vspace{-0.2cm}
    \begin{subfigure}{0.3\textwidth}
    \subcaption{} \label{subfig:right}
    \vspace{-3mm}
    \includegraphics[width=0.9\linewidth,height=1.4in]{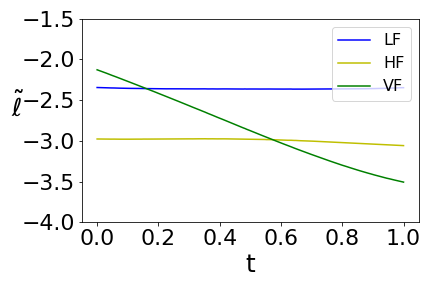}
    \end{subfigure}
    
    \caption{Posterior analysis for three synthetic datasets, LF, HF and VF. (\ref{subfig:upper-left}), (\ref{subfig:lower-left}) display the posterior predictive processes in Dimension 1 and Dimension 2, respectively. Blue, black and red dots refer to the training, testing and predictive data and the red dashed lines refer to $95\%$ credible bands. (\ref{subfig:right}) shows the estimates of log length-scale function via posterior mean.}
    \label{fig:pos_VSGPRN_syn}
\end{figure}

In the stochastic optimization, we use the learning rates of $0.005$ for all parameters with $2000$ epochs. In this experiment, we standardize both inputs and outputs and initialzie length-scale $\exp(2)$ for covariance function $K^l$ and length-scale $\exp(0)$ for covariance function $K^\ell$. We optimize all hyper-parameters in the optimization.

We quantify the model performance in terms of root mean square error (RMSE), average length of confidence interval (ALCI), and coverage rate (CR) on the test set. A smaller RMSE corresponds to better predictive performance of the model, and a smaller ALCI implies a smaller predictive uncertainty. As for CR, The better the model prediction performance is, the closer CR is to the percentile of the credible band. Those results are reported by the mean and standard deviation with 10 different random initializations of model parameters. Quantitative comparisons relating to all three datasets are in Table~\ref{tab:sim}. We compare with independent Gaussian process regression (IGPR) \citep{Rasmussen_2004}, the intrinsic coregionalization model (ICM) \citep{wackernagel2013multivariate}, Collaborative Multi-Output Gaussian Processes (CMOGP) \citep{nguyen2014collaborative} and variational inference of Gaussian process regression network \cite{nguyen2013efficient} on three synthetic datasets. In both CMOGP and VSGPRN approaches, we use $20$ inducing variables. We did not compare with the maximum a posteriori inference in the SGPRN model \cite{meng2021nonstationary} because the corresponding inference cannot handle missing data. 

We report the posterior predictive processes of Dimension 1 (Figure~\ref{subfig:upper-left}) and Dimension 2 (Figure~\ref{subfig:lower-left}) for datasets LF, HF and VF in Figure~\ref{fig:pos_VSGPRN_syn}. Blue, black and red dots refer to the training, testing and predictive data and the red dashed lines refer to $95\%$ credible bands. Comparing the predictive data between the two dimensions, we find that VSGPRN correctly learns the varying correlation from positive to negative. VSGPRN also displays the correct characteristics of the smoothness (Figure~\ref{subfig:right}) in three different datasets. Table~\ref{tab:sim} illustrates that VGPRN and VSGPRN have similar average predictive performs and in the varying frequency case, VGPRN performs better. That is because VSGPRN introduces inducing variables for sparse approaximation while VGPRN does not. Both VSGPRN and VGPRN significantly outperform other models, because they model the dependence of outputs. Compared with VGPRN, VSGPRN has a significantly smaller uncertainty of the prediction results, becasue of the smaller prediction standard deviation. That is because GPRN model is not identifiable on the mixing coefficients, which would lead to very sensitive prediction results that strongly depend on parameter initializations and thus this issue makes model interpretation meaningless. SGPRN has a weakly identifiable structure on the mixing coefficients and it would make inference more robust and render more meaningful interpretation on the data.

\subsection{Real Data Experiments}
We further examined model predictive performance on three real-world datasets. Due to the large size of real data, the standard Gaussian process models tested the in synthetic experiments cannot handle them. Therefore, we compare our model with two sparse Gaussian process models, i.e., independent sparse Gaussian process regression (ISGPR) \citep{Snelson_2006} and the sparse linear model of corregionalization (SLMC) \citep{wackernagel2013multivariate} implemented using the GPy package from the Sheffield machine learning group. Moreover, we explore the dynamics of correlation of neuronal activity visually with electorcorticography data.

\subsubsection{Environmental Time Series Data}
The first experiment is conducted on a PM2.5 dataset, coming from the UCI Machine Learning Repository \citep{liang_2015}. PM2.5 describes fine inhalable particles with diameters that are generally 2.5 micrometers and smaller and this dataset is hourly data containing the PM2.5 samples in five cities in China along with meteorological data, from Jan 1st, 2010 to Dec 31st, 2015.  We consider six important attributes: PM2.5 concentration (PM), dew point (DEWP), temperature (TEMP), humidity (HUMI), pressure (PRES) and cumulated wind speed (lws). In order to be able to compare with SLMC, which cannot run on the whole dataset, we use the first 5000 standardized multivariate records. Those records have 290 missing values and 29710 observed values. For each feature, we standardize by subtracting the mean value and dividing by its standard deviation. Finally, $20\%$ of data of PM are taken as testing data while the remaining are treated as training data. Thus, in the 5000 records, there are 28768 output variable observations in the training set, and 942 PM values in the test set.

We considered three independent experiments for VSGPRN with $50, 100$ and $200$ equispaced inducing inputs on time range $(0, 5000)$. The length-scale parameters were set to $\exp(10)$ for both $K^L$ and $K_\ell$, ran for $30$ epochs with learning rate $0.01$, and had batch size $1024$. For the comparators, we fit SGPR and SLMC models with $100$ equispaced inducing inputs. The root mean squared errors (RMSE) on the testing data are shown in Table~\ref{tab:real data results}, illustrating that VSGPRN had better prediction performance compared with the ISGPR and SLMC, even when using less inducing points. We show prediction results with different mini-batch sizes in the Appendix B.

\begin{table*}[ht!]
    \centering
    \caption{Empirical results for PM2.5 dataset and HCP dataset. Each model's performance is summarized by root mean square error on testing data (RMSE). The number of inducing points is given in parentheses.}
    \label{tab:real data results}
    \resizebox{\linewidth}{!}{
    \begin{tabular}{|c|c|c|c|c|c|}
    \hline
    Data & ISGPR (100) \citep{Snelson_2006} & SLMC (100) \citep{wackernagel2013multivariate} & VSGPRN (50) & VSGPRN (100) & VSGPRN (200) \\
    \hline
    PM2.5 & 0.994 & 0.948 & 0.840 & 0.708 & 0.625 \\
    \hline 
    HCP & 1.023 & - & 1.008 & 0.997 & 0.899 \\
    \hline 
    \end{tabular}
    }
\end{table*}

\subsubsection{Resting-State Functional MRI Data}
The second experiment explores the functional connectivity of the brain, using a publicly available resting-state functional MRI (rs-fMRI) database obtained from the Human Connectome Project (HCP) S12000 data release \citep{smith2013resting} for 812 subjects. The HCP pre-processing pipeline~\citep{wu20171200} yielded one representative time series across 4800 time points per independent component analysis (ICA) component for each subject at several different dimensionalities. We used the rs-fMRI timeseries from 15 ICA components with a random subject ID 990366 in this experiment. Specifically, we standardized the ICA components by subtracting the mean value of each feature and dividing it by its standard deviation. $20\%$ of data in the first component are treated as the testing data while the remaining are treated as training data. Then we have 71040 training data and 960 testing data. 

We conducted the experiments using the same models used for the PM2.5 dataset. Because the SLMC model does not scale well, the prediction result is not available via our computing resource. Therefore, we only compared our results on VSGPRN and ISGPR for the HCP dataset in Table~\ref{tab:real data results}. In this experiment, VSGPRN sets the length-scale parameters to $\exp(5)$ for both $K^L$ and $K_\ell$ and run $10$ epochs with learning rate $0.01$ and batch size $1024$. Table~\ref{tab:real data results} shows that as the number of inducing points increases, the prediction performance improves and our model always outperforms the ISGPR. Even when we only take 50 inducing inputs for VSGPRN, the prediction result is still better than that for ISGPR with 100 inducing inputs. In the Appendix B, we report the prediction results with different mini-batch sizes.

\subsubsection{Electrocorticography Data}
Finally, we evaluate our model on Electrocorticography (ECoG) data collected in the Bouchard Lab\cite{dougherty2019laminar}. High-gamma activity from ECoG is a commonly-used signal containing the majority of task relevant information for understanding the human brain  \citep{livezey2019deep}, and the experiments in \cite{dougherty2019laminar} record $\mu$ECoG cortical surface electrical potentials (CSEPs) from 128 channels and demonstrate that stimulus evoked CSEPs carry a multi-modal frequency response  peaking in the $H\gamma$ range (70-170Hz). We selected a 5-second time interval and extracted the z-scored high gamma band from 25 channels on $5 \times 5$ subgrid located at the center of the $16 \times 8$ grid. For each channel, the records are sampled at $400Hz$ and thus we have $2000$ time points. To illustrate the data, we plotted the functional boxplot \citep{sun2011functional,meng2017growth} for the time series of the 25 channels in Figure~\ref{fig:ECOG}. Next we conduct two experiments, experiment $\mathcal{E}_1$ for data within 2 seconds and experiment $\mathcal{E}_2$ for data within 5 seconds. We run $40$ epochs for each experiment. To keep the number of iterations within each epoch for both experiments close, we set the batch size as $512$ for $\mathcal{E}_1$ and $2048$ for $\mathcal{E}_2$. 

\begin{figure}[!ht]
    \centering
    \begin{subfigure}{0.45\linewidth}
    \subcaption{} \label{subfig:z-scores}
    \vspace{-3mm}
    \includegraphics[width = \linewidth, height = 30mm]{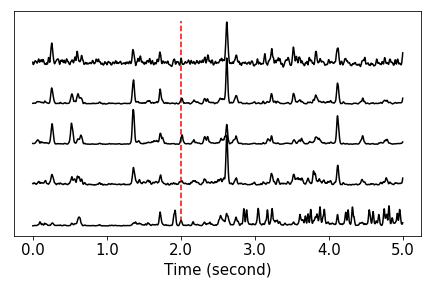}
    \end{subfigure}
    \begin{subfigure}{0.45\linewidth}
    \subcaption{} \label{subfig:z-score_fboxplot}
    \vspace{-3mm}
    \includegraphics[width = \linewidth, height = 30mm]{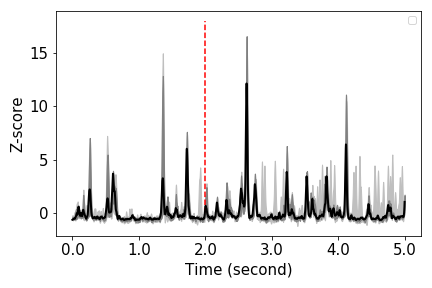}
    \end{subfigure}
    \hspace*{\fill}%
    \begin{subfigure}{0.45\linewidth}
    \subcaption{} \label{subfig:E1_pos}
    \vspace{-3mm}
    \includegraphics[width = \linewidth, height = 30mm]{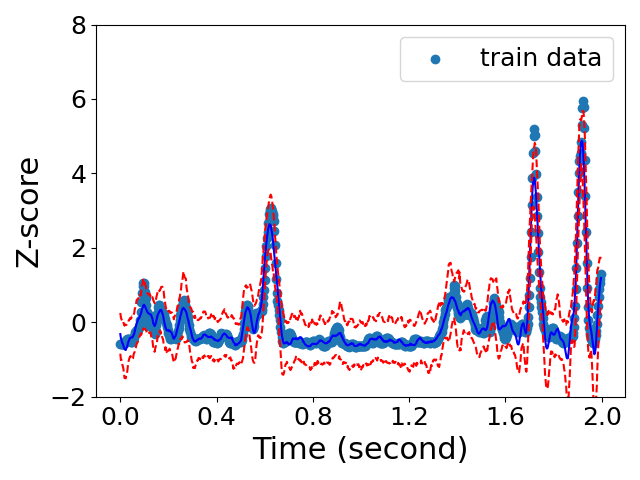}
    \end{subfigure}
    \begin{subfigure}{0.45\linewidth}
    \subcaption{} \label{subfig:E2_pos}
    \vspace{-3mm}
    \includegraphics[width = \linewidth, height = 30mm]{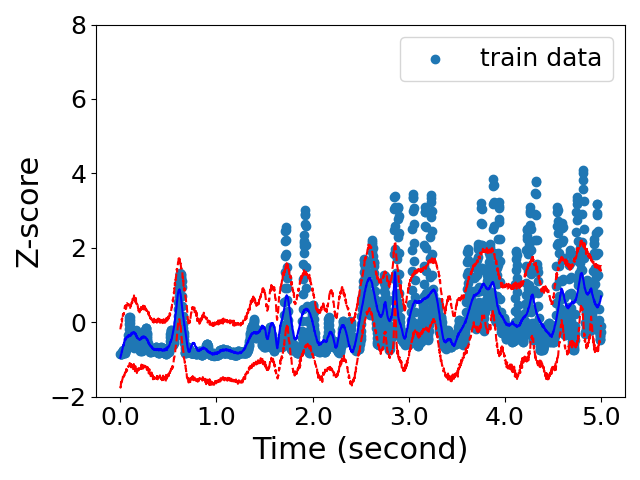}
    \end{subfigure}
    \hspace*{\fill}%
    \begin{subfigure}{0.9\linewidth}
    \subcaption{} \label{subfig:E1_corrs}
    \vspace{-3mm}
    \includegraphics[width = \linewidth, height = 30mm]{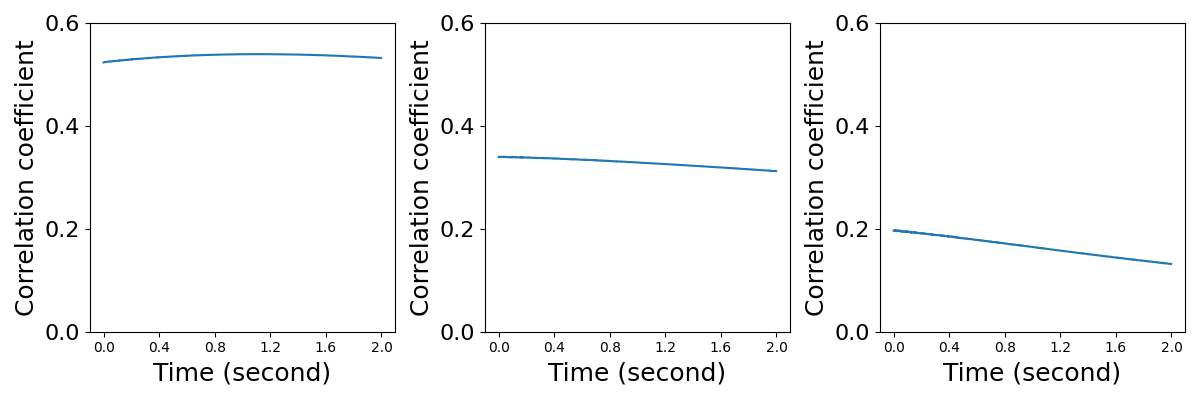}
    \end{subfigure}
    \hspace*{\fill}%
    \begin{subfigure}{0.9\linewidth}
    \subcaption{} \label{subfig:E2_corrs}
    \vspace{-3mm}
    \includegraphics[width = \linewidth, height = 30mm]{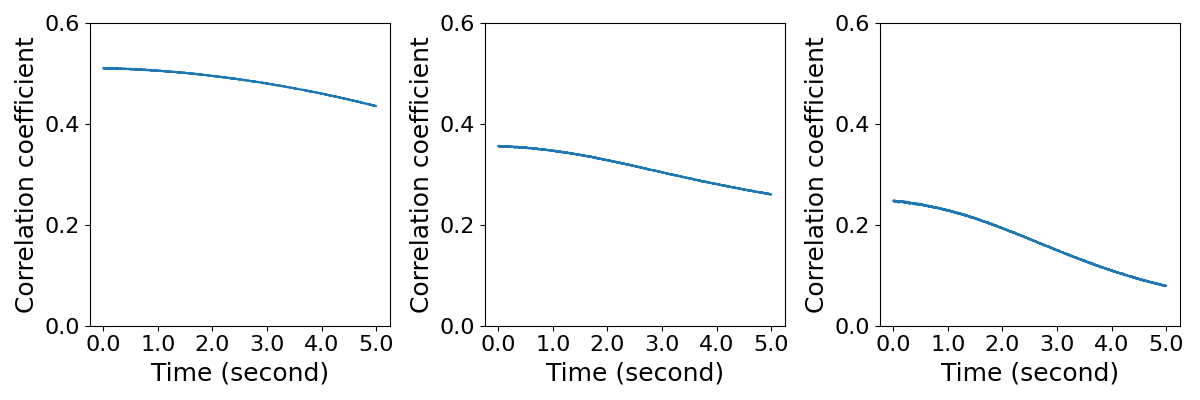}
    \end{subfigure}
    \caption{Information of the ECoG data and corresponding posterior analysis for experiment $\mathcal{E}_1$ and experiment $\mathcal{E}_2$. (\ref{subfig:z-scores}) refers to the z-score curves for the five channels on the diagonal of grids. (\ref{subfig:z-score_fboxplot}) provides the functional boxplot of all z-score curves on the grids. (\ref{subfig:E1_pos}) and (\ref{subfig:E2_pos}) show the posterior predictive process for the centered channel, and (\ref{subfig:E1_corrs}) and (\ref{subfig:E2_corrs}) show the averaged correlation processes for distance $d = 1, 2, 3$ for experiment $\mathcal{E}_1$ and experiment $\mathcal{E}_2$ respectively.}
    \label{fig:ECOG}
\end{figure}

Since ECoG data are high-frequency sampled and z-scores in the frequency domain are smooth, the naive interpolation can achieve high predictive accuracy and thus prediction task is not of interest. So we report the prediction analysis in the Appendix B. In this section, our interest is to explore the time-varying cross-correlation across channels. We trained the whole data in experiment $\mathcal{E}_1$ and experiment $\mathcal{E}_2$. Assume we treat the distance between two consecutive samples as one unit, i.e., $t_{i+1}- t_{i} = 1$. We considered $50$ fixed equally spaced inducing inputs on the time interval. We set the length-scale parameters to $\exp(10)$ for $K^L$ by assuming that the coefficient should smoothly change across time. For the length-scale parameters for $K_\ell$, we assumed that the length-scale function $\ell$ is flexible and less smooth and set $\exp(5)$ for $\mathcal{E}_1$ and $\exp(2)$ for $\mathcal{E}_2$. This is because the distances between the consecutive inducing inputs in $\mathcal{E}_1$ are $\mathcal{E}_2$ are $16$ and $40$ and the default hyper-parameters guarantee the GP can reasonably learn the dependence. 

We show the posterior predictive processes for one channel from two experiments in Figure~\ref{fig:ECOG}. It shows that the 95\% credible interval of the posterior predictive process includes almost $95\%$ observations in $\mathcal{E}_1$, while the 95\% credible interval of the posterior predictive process in $\mathcal{E}_2$ performs worse and cannot peaks or spikes. On the other hand, considering the same number of inducing inputs, as the length of time series increases, the prediction performance becomes worse. This is because as the length of time series increases, the same number of inducing variables are required to summarize more nonstationary information. This causes the model to have more difficulty simultaneously being sensitive to the local information, making the prediction process more smooth as shown in Figure~\ref{fig:ECOG}. To model local information, more inducing points are required.

We estimated the correlation process for pairwise channels using the posterior mean and took the average of those processes for which the distance between all pairs of channels with a constant $d$. We plot the averaged correlation processes for distance $d = 1, 2, 3$ in Figure~\ref{fig:ECOG}. The resulting correlation processes from the two experiments show that the dynamical behavior in $\mathcal{E}_2$ in the first 2 seconds is consistent with that in $\mathcal{E}_1$. It implies our model estimates are robust to the length of the time series. Moreover, as the distance between pair of channels increases, the correlation decreases. This is in agreement with the known neurobiology.

\section{Conclusions} \label{sec:conclusion}
We propose a novel variational inference approach for structured Gaussian process regression network named variational structured Gaussian process regression network (VSGPRN). We introduce inducing variables and proposed structured variational distribution to reduce computational burden. Moverover, we perform exact marginalization of latent variables and construct tractable lower bound of log likelihood to allow it suitable for doubly stochastic inference. In our method, the computation complexity is independent of the size of inputs and outputs. Compared with the likelihood based inference for SGPRN model in \cite{meng2021nonstationary}, our model provides a natural extension to missing data cases and succeeds in handling time-varying correlations under different scenarios. We also show that VSGPRN achieves better imputation performance on missing data than state-of-the-art models in synthetic experiments and real-world data experiments. Moreover, we provide an estimation of the correlation of outputs across input domains, as demonstrated in the ECoG experiment. We find that as the distance between two channels increase the correlation decreases, which is in agreement with the known neurobiology. In the future, we will explore extending the VSGPRN model to incorporate the effects of exogenous variables, which will better model the ECoG experiments, in which the animal was being presented with auditory stimuli, which drives the recorded neural activity.

\appendix
\section{Derivations in Model inference}

\subsection{Derivations for prior distributions}

Given the model specification, we have the conditional distributions in (8), (9) and (10). The conditional mean and covariance matrices are derived as follows:
\begin{align}
    \bm \mu_{ij}^l & = K^l(\bm X, \bm Z)K^l(\bm Z, \bm Z)^{-1}\bm u_{ij} \,, \nonumber \\
    \bm \mu_{d}^g & = K^g(\bm X, \bm Z)K^g(\bm Z, \bm Z)^{-1}\bm w_{d} \,, \nonumber \\
    \bm \mu^\ell & = K^\ell(\bm X, \bm Z)K^\ell(\bm Z, \bm Z)^{-1}\bm v \,, \nonumber \\
    \Sigma_{ij}^l & = K^l(\bm X, \bm X) -  K^l(\bm X, \bm Z)K^l(\bm Z, \bm Z)^{-1}K^l(\bm Z, \bm X) \,, \nonumber\\
    \Sigma_{ij}^g & = K^g(\bm X, \bm X) -  K^g(\bm X, \bm Z)K^g(\bm Z, \bm Z)^{-1}K^g(\bm Z, \bm X) \,, \nonumber \\
    \Sigma_{ij}^\ell & = K^\ell(\bm X, \bm X) -  K^\ell(\bm X, \bm Z)K^\ell(\bm Z, \bm Z)^{-1}K^\ell(\bm Z, \bm X) \,. \nonumber
\end{align}

\subsection{Derivations for variational distributions}

We claim that the conditional mean and covariance matrices are shown as follows:

\begin{align}
    \tilde{\bm \mu}_{ij}^l & = K^l(\bm X, \bm Z)K^l(\bm Z, \bm Z)^{-1} \bm m_{ij}^u \,, \nonumber \\
    \tilde{\bm \mu}_d^g & = K^g(\bm X, \bm Z)K^g(\bm Z, \bm Z)^{-1} \bm m_{ij}^w \,, \nonumber \\
    \tilde{\Sigma}_{ij}^l & = K^l(\bm X, \bm X) - K^l(\bm X, \bm Z)K^l(\bm Z, \bm Z)^{-1}K^l(\bm Z, \bm X) \nonumber \\
    & + K^l(\bm X, \bm Z)K^l(\bm Z, \bm Z)^{-1}S_{ij}^uK^l(\bm Z, \bm Z)^{-1}K^l(\bm Z, \bm X) \,, \nonumber \\
    \tilde{\Sigma}_d^g & = K^g(\bm X, \bm X) - K^g(\bm X, \bm Z)K^g(\bm Z, \bm Z)^{-1}K^g(\bm Z, \bm X) \nonumber \\
    & + K^g(\bm X, \bm Z)K^g(\bm Z, \bm Z)^{-1}S_{ij}^w K^g(\bm Z, \bm Z)^{-1}K^g(\bm Z, \bm X) \,.\nonumber
\end{align}

This result comes from the Lemma~\ref{lemma1}.

\begin{lemma}\label{lemma1}
    Suppose $\bm y| \bm x \sim \mathcal{N}(K_{12}K_{22}^{-1}\bm x, K_{11} - K_{12}K_{22}^{-1}K_{21})$ and $\bm x \sim \mathcal{N}(\bm \mu, \Sigma)$. The marginalized distribution of $Y$ is 
    \begin{align}
        \bm y \sim \mathcal{N}(\tilde{\bm \mu}, \tilde{\Sigma})\,, \nonumber
    \end{align}
    where $\tilde{\bm \mu} = K_{12}K_{22}^{-1}\bm \mu$ and  $\tilde{\Sigma} = K_{11} - K_{12}K_{22}^{-1}K_{21} + K_{12}K_{22}^{-1}\Sigma K_{22}^{-1}K_{21}$. And the marginal distribution of $y_n$ is 
    \begin{align}
        y_n \sim \mathcal{N}(\tilde{\mu}_n, \tilde{\sigma}^2_n) \,, \nonumber
    \end{align}
    where $\tilde{\mu}_n = \tilde{k}_n^TK_{22}\bm u$ and $\tilde{\sigma}^2_n = \sigma^2_n - \tilde{k}_n^TK_{22}\tilde{k}_n^T + \tilde{k}_n^TK_{22}^{-1}\Sigma K_{22}^{-1}\tilde{k}_n$. $\tilde{k}_n^T$ is the $n^{\text{th}}$ row of $K_{12}$ and $\tilde{\sigma}^2_n$ is the $i^{\text{t}}$ element of the diagonal of $K_{11}$.
\end{lemma}

\subsection{Derivations in the computation of ELBO}
We first introduce Lemma~\ref{lemma2} as follows

\begin{lemma}\label{lemma2}
    Suppose $\bm y| \bm f \sim \mathcal{N}(\bm y| \bm X\bm f, \sigma^2\bm I)$ and $\bm f \sim \mathcal{N}(\bm f| \bm u, \bm \Sigma)$. Then
    \begin{align}
        \int \log\mathcal{N}(\bm y| \bm X\bm f, \sigma^2\bm I)\mathcal{N}(\bm f| \bm \mu, \bm \Sigma)d\bm f = \log\mathcal{N}(\bm y| \bm X\bm \mu, \sigma^2\bm I) -\frac{1}{2\sigma^2}\mathrm{tr}(\bm X^T\bm X\Sigma).
    \end{align}
\end{lemma}

\begin{proof}
Let the dimension of $\bm y$ be $n$, then
\begin{align}
    & \int \log\mathcal{N}(\bm y| \bm X\bm f, \sigma^2\bm I)\mathcal{N}(\bm f| \bm \mu, \bm \Sigma)d\bm f \nonumber \\
    & = E_{p(\bm f|\bm \mu, \bm \Sigma)}(-\frac{n}{2}\log(2\pi\sigma^2) - \frac{1}{2\sigma^2}(\bm y^T\bm y - 2\bm y^T(\bm X\bm f) + \bm f^T\bm X^T \bm X\bm f)) \nonumber \\
    & = -\frac{n}{2}\log(2\pi\sigma^2) - \frac{1}{2\sigma^2}\mathrm{tr}\left(\bm y \bm y^T - 2 \bm X\bm u\bm y^T + \bm X^T \bm X(\bm \mu\bm \mu^T + \Sigma)\right) \nonumber \\
    & = \log\mathcal{N}(\bm y| \bm X\bm \mu, \sigma^2 \bm I) -\frac{1}{2\sigma^2}\mathrm{tr}(\bm X^T\bm X\Sigma) \,. \nonumber 
\end{align}
\end{proof}

Then according to Lemma~\ref{lemma2}, the individual expectation can be rewritten as
\begin{align}
    & \mathrm{E}_{q(\bm g_n, \bm l_n)}\log(p(y_{nd}|\bm g_n, \bm l_n)) \nonumber \\
    & = \int \log\mathcal{N}(y_{nd}| \bm l_{d\cdot n}^T \bm g_{n}, \sigma_{err}^2)\mathcal{N}(\bm g_n| \tilde{\bm \mu}^g_{n}, \mathrm{diag}(\tilde{\sigma}^{g2}_{1n}, \cdots, \tilde{\sigma}^{g2}_{Dn}))\nonumber \\ 
    & \quad q(\tilde{\bm \ell}, \bm v)q(\bm l_{d\cdot n})d(\bm l_{d\cdot n}, \bm g_n, \bm \tilde{\bm \ell} ,\bm v) \nonumber \\
    & = \int\left(\log\mathcal{N}(y_{nd}|\sum_{j=1}^D l_{djn}\tilde{\bm \mu}^g_{n}, \sigma^2_{err}) - \frac{1}{2\sigma^2_{err}}\sum_{j=1}^D l_{djn}^2\tilde{\sigma}^{g2}_{jn}\right) \nonumber \\
    & \quad q(\tilde{\bm \ell}, \bm v)q(\bm l_{d\cdot n})d(\bm l_{d\cdot n}, \tilde{\bm \ell}, \bm v)\,. \nonumber
\end{align}

\subsection{Derivations for reparameterization}
The re-parameterization is proposed for the distribution  $q(\ell_n, \bm v)$, $q(\bm g_n| \ell_n, \bm v)$ and $q(\bm l_{d\cdot n})$.

\begin{itemize}
    \item As for $q(\ell_n, \bm v)$, we have
    \begin{align}
        \bm v & = \bm m^v + S^{v\frac{1}{2}}\bm z^v \,, \nonumber \\
        \ell_n & = \exp\Big(K^\ell(\bm x_n, \bm Z)K^\ell(\bm Z, \bm Z)^{-1} \bm v + \Big(K^\ell(\bm x_n, \bm x_n) \nonumber \\
        & \quad
        - K^\ell(\bm x_n, \bm Z)K^{\ell}(\bm Z, \bm Z)^{-1}K^{\ell}(\bm Z,\bm x_n) \Big)^{\frac{1}{2}}z_n^\ell \Big)\,, \nonumber
    \end{align}
    where $\bm z^v \sim \mathcal{N}(\bm 0, I)$ and $\bm z_n^\ell \sim \mathcal{N}(0, 1)$.
    
    \item As for $q(\bm g_n| \ell_n, \bm v)$, we have 
    \begin{align}
        g_{dn} & = \tilde{\mu}_{dn}^g + \tilde{\sigma}_{dn}^g z_{dn}^g \nonumber
    \end{align}
    where $z_{dn}^g \sim \mathcal{N}(0,1)$.
    
    \item Finally, as for $q(\bm l_{d\cdot n})$, we have
    \begin{align}
        l_{ijn} = \begin{cases}
                      \tilde{\mu}^l_{iin} + \tilde{\sigma}^l_{iin}z^l_{iin} & i = j \\
                      \exp\left(\tilde{\mu}^l_{ijn} + \tilde{\sigma}^l_{ijn}z^l_{ijn}\right) & i > j
                  \end{cases} \nonumber
    \end{align}
    where $z^l_{ijn} \stackrel{iid}{\sim} \mathcal{N}(0,1)$. 
\end{itemize}

\section{ECoG experiments}

\subsection{Prediction performance}
 In experiment $\mathcal{E}_1$, to compare the models' performance, we randomly took $20\%$ of data in the centered channel as testing data and took the remaining as training data. It implies we have $160$ samples in the centered channel in the training set and $19840$ samples in the testing set. The root mean square errors of testing data are reported. The RMSEs for ISGPR(100), SLMC(100), VSGPRN(50), VSGPRN(100) and VSGPRN(200) are $0.311$, $0.446$, $0.707$, $0.658$ and $0.657$. The number in the bracket refers to the number of inducing points. The predictive results show that our model cannot compete ISGPR and SLMC models. The reason is because that ECoG data is smooth with negligible noise in each channel and it is easy to predict the testing data using the nearby data within the channel. Learning the cross-correlation in ECoG data using VSGPRN approach does not contribute to better prediction result. Moreover, the inference in VSGPRN approach overestimate the variance of noises and causes under-fitting result in ECoG data . However, VSGPRN approach can provide estimates of the time-varying correlation for pairwise channels while other models cannot.

\subsection{Prediction result under different mini-batch sizes}

We explored how the mini-batch size $B$ affects the prediction result in datasets, PM2.5 and HCP. Specifically, considering different number of mini-batch sizes, we plotted the RMSEs on the testing data during the training process in Figure~\ref{fig:RMSE_BS}. 

For both datasets, Figure~\ref{fig:RMSE_BS} illustrates that RMSEs would converge to the same value as training time increases. 
In the PM2.5 data, the prediction performance monotonically improves with time increasing. However, in the HCP data, the RMSEs with different mini-batch sizes converge differently. When the batch size increases, the prediction performance becomes better. Empirical results for the PM2.5 data and HCP data suggest that the mini-batch size may affect the predictive performance in practice. The behavior depends on the characteristics of data.

\begin{figure}[ht!]
    \centering
    \includegraphics[width=0.45\linewidth]{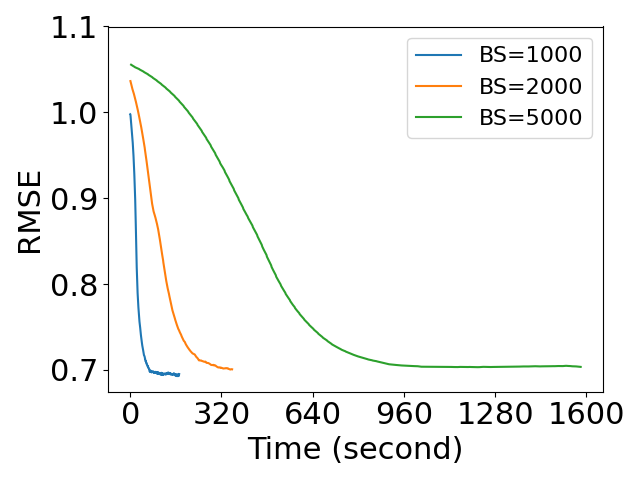}
    \includegraphics[width=0.45\linewidth]{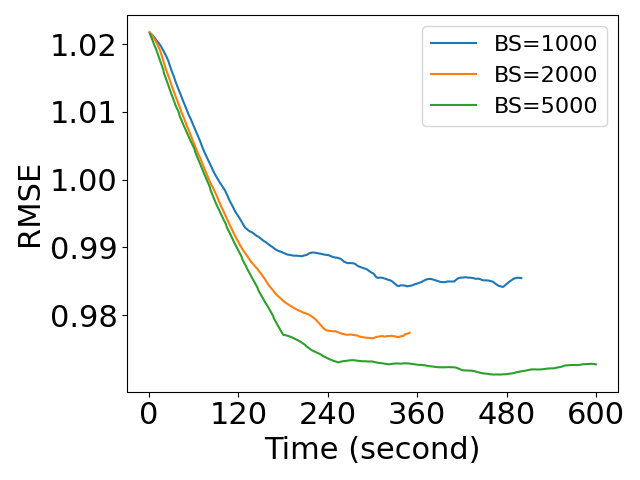}
    \caption{The root mean squared error on the testing data for two datasets i.e. PM2.5 and HCP. All experiments are conducted using VSGPRN method with the same $100$ equispace inducing inputs but considering different mini batch-sizes (BS).}
    \label{fig:RMSE_BS}
\end{figure}




\clearpage

\bibliographystyle{numcompress}
\bibliography{mybibfile}

\begin{thebibliography}{10}
\expandafter\ifx\csname url\endcsname\relax
  \def\url#1{\texttt{#1}}\fi
\expandafter\ifx\csname urlprefix\endcsname\relax\def\urlprefix{URL }\fi
\expandafter\ifx\csname href\endcsname\relax
  \def\href#1#2{#2} \def\path#1{#1}\fi

\bibitem{durichen2014multitask}
R.~D{\"u}richen, M.~A. Pimentel, L.~Clifton, A.~Schweikard, D.~A. Clifton,
  Multitask gaussian processes for multivariate physiological time-series
  analysis, IEEE Transactions on Biomedical Engineering 62~(1) (2014) 314--322.

\bibitem{burnham1999latent}
A.~J. Burnham, J.~F. MacGregor, R.~Viveros, Latent variable multivariate
  regression modeling, Chemometrics and Intelligent Laboratory Systems 48~(2)
  (1999) 167--180.

\bibitem{sanchez2004svm}
M.~S{\'a}nchez-Fern{\'a}ndez, M.~de~Prado-Cumplido, J.~Arenas-Garc{\'\i}a,
  F.~P{\'e}rez-Cruz, Svm multiregression for nonlinear channel estimation in
  multiple-input multiple-output systems, IEEE transactions on signal
  processing 52~(8) (2004) 2298--2307.

\bibitem{alvarez2010efficient}
M.~{\'A}lvarez, D.~Luengo, M.~Titsias, N.~D. Lawrence, Efficient multioutput
  gaussian processes through variational inducing kernels, in: Proceedings of
  the Thirteenth International Conference on Artificial Intelligence and
  Statistics, 2010, pp. 25--32.

\bibitem{alvarez2011computationally}
M.~A. {\'A}lvarez, N.~D. Lawrence, Computationally efficient convolved multiple
  output gaussian processes, The Journal of Machine Learning Research 12 (2011)
  1459--1500.

\bibitem{bourgault1991multivariable}
G.~Bourgault, D.~Marcotte, Multivariable variogram and its application to the
  linear model of coregionalization, Mathematical Geology 23~(7) (1991)
  899--928.

\bibitem{goulard1992linear}
M.~Goulard, M.~Voltz, Linear coregionalization model: tools for estimation and
  choice of cross-variogram matrix, Mathematical Geology 24~(3) (1992)
  269--286.

\bibitem{ver1998constructing}
J.~M. Ver~Hoef, R.~P. Barry, Constructing and fitting models for cokriging and
  multivariable spatial prediction, Journal of Statistical Planning and
  Inference 69~(2) (1998) 275--294.

\bibitem{ver2004flexible}
J.~M. Ver~Hoef, N.~Cressie, R.~P. Barry, Flexible spatial models for kriging
  and cokriging using moving averages and the fast fourier transform (fft),
  Journal of Computational and Graphical Statistics 13~(2) (2004) 265--282.

\bibitem{gneiting2010matern}
T.~Gneiting, W.~Kleiber, M.~Schlather, Mat{\'e}rn cross-covariance functions
  for multivariate random fields, Journal of the American Statistical
  Association 105~(491) (2010) 1167--1177.

\bibitem{apanasovich2010cross}
T.~V. Apanasovich, M.~G. Genton, Cross-covariance functions for multivariate
  random fields based on latent dimensions, Biometrika 97~(1) (2010) 15--30.

\bibitem{gelfand2004nonstationary}
A.~E. Gelfand, A.~M. Schmidt, S.~Banerjee, C.~Sirmans, Nonstationary
  multivariate process modeling through spatially varying coregionalization,
  Test 13~(2) (2004) 263--312.

\bibitem{wilson2011gaussian}
A.~G. Wilson, D.~A. Knowles, Z.~Ghahramani, Gaussian process regression
  networks, arXiv preprint arXiv:1110.4411.

\bibitem{kleiber2012nonstationary}
W.~Kleiber, D.~Nychka, Nonstationary modeling for multivariate spatial
  processes, Journal of Multivariate Analysis 112 (2012) 76--91.

\bibitem{meng2021nonstationary}
R.~Meng, B.~Soper, H.~K. Lee, V.~X. Liu, J.~D. Greene, P.~Ray, Nonstationary
  multivariate gaussian processes for electronic health records, Journal of
  Biomedical Informatics 117 (2021) 103698.

\bibitem{titsias2010bayesian}
M.~Titsias, N.~D. Lawrence, Bayesian gaussian process latent variable model,
  in: Proceedings of the Thirteenth International Conference on Artificial
  Intelligence and Statistics, 2010, pp. 844--851.

\bibitem{seeger2005semiparametric}
M.~Seeger, Y.-W. Teh, M.~Jordan, Semiparametric latent factor models, in:
  AISTATS, 2005.

\bibitem{bonilla2008multi}
E.~V. Bonilla, K.~M. Chai, C.~Williams, Multi-task gaussian process prediction,
  in: Advances in neural information processing systems, 2008, pp. 153--160.

\bibitem{wackernagel2013multivariate}
H.~Wackernagel, Multivariate geostatistics: an introduction with applications,
  Springer Science \& Business Media, 2013.

\bibitem{titsias2011spike}
M.~Titsias, M.~L{\'a}zaro-Gredilla, Spike and slab variational inference for
  multi-task and multiple kernel learning, Advances in neural information
  processing systems 24 (2011) 2339--2347.

\bibitem{nguyen2013efficient}
T.~Nguyen, E.~Bonilla, Efficient variational inference for gaussian process
  regression networks, in: Artificial Intelligence and Statistics, 2013, pp.
  472--480.

\bibitem{li2020scalable}
S.~Li, W.~Xing, M.~Kirby, S.~Zhe, Scalable variational gaussian process
  regression networks, arXiv preprint arXiv:2003.11489.

\bibitem{meng2021bayesian}
R.~Meng, K.~Bouchard, Bayesian inference in high-dimensional time-serieswith
  the orthogonal stochastic linear mixing model, arXiv preprint
  arXiv:2106.13379.

\bibitem{guhaniyogi2013modeling}
R.~Guhaniyogi, A.~O. Finley, S.~Banerjee, R.~K. Kobe, Modeling complex spatial
  dependencies: Low-rank spatially varying cross-covariances with application
  to soil nutrient data, Journal of Agricultural, Biological, and Environmental
  Statistics 18~(3) (2013) 274--298.

\bibitem{Hensman_2013}
J.~Hensman, N.~Fusi, N.~D. Lawrence,
  \href{http://dl.acm.org/citation.cfm?id=3023638.3023667}{Gaussian processes
  for big data}, in: Proceedings of the Twenty-Ninth Conference on Uncertainty
  in Artificial Intelligence, UAI'13, AUAI Press, Arlington, Virginia, United
  States, 2013, pp. 282--290.
\newline\urlprefix\url{http://dl.acm.org/citation.cfm?id=3023638.3023667}

\bibitem{nguyen2014collaborative}
T.~V. Nguyen, E.~V. Bonilla, et~al., Collaborative multi-output gaussian
  processes., in: UAI, 2014, pp. 643--652.

\bibitem{salimbeni2017doubly}
H.~Salimbeni, M.~Deisenroth, Doubly stochastic variational inference for deep
  gaussian processes, in: Advances in Neural Information Processing Systems,
  2017, pp. 4588--4599.

\bibitem{wang2020doubly}
Q.~Wang, H.~Van~Hoof, Doubly stochastic variational inference for neural
  processes with hierarchical latent variables, in: International Conference on
  Machine Learning, PMLR, 2020, pp. 10018--10028.

\bibitem{moller1998log}
J.~M{\o}ller, A.~R. Syversveen, R.~P. Waagepetersen, Log gaussian cox
  processes, Scandinavian journal of statistics 25~(3) (1998) 451--482.

\bibitem{remes2017non}
S.~Remes, M.~Heinonen, S.~Kaski, Non-stationary spectral kernels, arXiv
  preprint arXiv:1705.08736.

\bibitem{ranganath2014black}
R.~Ranganath, S.~Gerrish, D.~Blei, Black box variational inference, in:
  Artificial intelligence and statistics, PMLR, 2014, pp. 814--822.

\bibitem{kingma2013auto}
D.~P. Kingma, M.~Welling, Auto-encoding variational bayes, arXiv preprint
  arXiv:1312.6114.

\bibitem{rezende2014stochastic}
D.~J. Rezende, S.~Mohamed, D.~Wierstra, Stochastic backpropagation and
  approximate inference in deep generative models, in: International Conference
  on Machine Learning, 2014, pp. 1278--1286.

\bibitem{titsias2014doubly}
M.~Titsias, M.~L{\'a}zaro-Gredilla, Doubly stochastic variational bayes for
  non-conjugate inference, in: International conference on machine learning,
  2014, pp. 1971--1979.

\bibitem{ghosh2018structured}
S.~Ghosh, J.~Yao, F.~Doshi-Velez, Structured variational learning of bayesian
  neural networks with horseshoe priors, arXiv preprint arXiv:1806.05975.

\bibitem{hoffman2010online}
M.~Hoffman, F.~R. Bach, D.~M. Blei, Online learning for latent dirichlet
  allocation, in: advances in neural information processing systems, 2010, pp.
  856--864.

\bibitem{hoffman2013stochastic}
M.~D. Hoffman, D.~M. Blei, C.~Wang, J.~Paisley, Stochastic variational
  inference, The Journal of Machine Learning Research 14~(1) (2013) 1303--1347.

\bibitem{Rasmussen_2004}
C.~Rasmussen, M.~Kuss, Gaussian processes in reinforcement learning, in:
  Advances in Neural Information Processing Systems 16,
  Max-Planck-Gesellschaft, MIT Press, Cambridge, MA, USA, 2004, pp. 751--759.

\bibitem{Snelson_2006}
E.~Snelson, Z.~Ghahramani,
  \href{http://papers.nips.cc/paper/2857-sparse-gaussian-processes-using-pseudo-inputs.pdf}{Sparse
  gaussian processes using pseudo-inputs}, in: Y.~Weiss, B.~Sch\"{o}lkopf,
  J.~C. Platt (Eds.), Advances in Neural Information Processing Systems 18, MIT
  Press, 2006, pp. 1257--1264.
\newline\urlprefix\url{http://papers.nips.cc/paper/2857-sparse-gaussian-processes-using-pseudo-inputs.pdf}

\bibitem{liang_2015}
X.~Liang, T.~Zou, B.~Guo, S.~Li, H.~Zhang, S.~Zhang, H.~Huang, S.~X. Chen,
  \href{https://royalsocietypublishing.org/doi/abs/10.1098/rspa.2015.0257}{Assessing
  beijing's pm<sub>2.5</sub> pollution: severity, weather impact, apec and
  winter heating}, Proceedings of the Royal Society A: Mathematical, Physical
  and Engineering Sciences 471~(2182) (2015) 20150257.
\newblock \href
  {http://arxiv.org/abs/https://royalsocietypublishing.org/doi/pdf/10.1098/rspa.2015.0257}
  {\path{arXiv:https://royalsocietypublishing.org/doi/pdf/10.1098/rspa.2015.0257}},
  \href {http://dx.doi.org/10.1098/rspa.2015.0257}
  {\path{doi:10.1098/rspa.2015.0257}}.
\newline\urlprefix\url{https://royalsocietypublishing.org/doi/abs/10.1098/rspa.2015.0257}

\bibitem{smith2013resting}
S.~M. Smith, C.~F. Beckmann, J.~Andersson, E.~J. Auerbach, J.~Bijsterbosch,
  G.~Douaud, E.~Duff, D.~A. Feinberg, L.~Griffanti, M.~P. Harms, et~al.,
  Resting-state fmri in the human connectome project, Neuroimage 80 (2013)
  144--168.

\bibitem{wu20171200}
H.~WU-Minn, 1200 subjects data release reference manual, URL https://www.
  humanconnectome. org.

\bibitem{dougherty2019laminar}
M.~E. Dougherty, A.~P. Nguyen, V.~L. Baratham, K.~E. Bouchard, Laminar origin
  of evoked ecog high-gamma activity, in: 2019 41st Annual International
  Conference of the IEEE Engineering in Medicine and Biology Society (EMBC),
  IEEE, 2019, pp. 4391--4394.

\bibitem{livezey2019deep}
J.~A. Livezey, K.~E. Bouchard, E.~F. Chang, Deep learning as a tool for neural
  data analysis: speech classification and cross-frequency coupling in human
  sensorimotor cortex, PLoS computational biology 15~(9) (2019) e1007091.

\bibitem{sun2011functional}
Y.~Sun, M.~G. Genton, Functional boxplots, Journal of Computational and
  Graphical Statistics 20~(2) (2011) 316--334.

\bibitem{meng2017growth}
R.~Meng, S.~Saade, S.~Kurtek, B.~Berger, C.~Brien, K.~Pillen, M.~Tester,
  Y.~Sun, Growth curve registration for evaluating salinity tolerance in
  barley, Plant methods 13~(1) (2017) 18.

\end{thebibliography}

\end{document}